\documentclass[11pt]{article}

\usepackage{fullpage}
\usepackage{mathpazo}
\usepackage[english]{babel}

\usepackage[utf8]{inputenc} 
\usepackage[T1]{fontenc}    
\usepackage{hyperref}       
\usepackage{csquotes}
\hypersetup{
colorlinks,
linkcolor={red!70!black},
citecolor={blue!70!black},
urlcolor={blue!80!black}
}
\usepackage{url}            
\usepackage{booktabs}       
\usepackage{amsfonts}       
\usepackage{nicefrac}       
\usepackage{microtype}      
\usepackage{setspace}

\usepackage{wrapfig}

\usepackage{amsmath,amssymb,amsthm,bbm,mathtools}

\usepackage{cleveref}
\creflabelformat{equation}{#2#1#3}

\usepackage[dvipsnames]{xcolor}
\usepackage{algorithm}
\usepackage{algorithmic}

\usepackage{enumerate}\usepackage[inline]{enumitem}
\usepackage{etoolbox,siunitx}
\robustify\bfseries
\usepackage{multirow,makecell}
\usepackage{threeparttable}

\usepackage[
 backend=biber,
 style=numeric,
 giveninits=true,
 maxcitenames = 10,
 maxbibnames = 10,
 ]{biblatex}
 \DeclareSourcemap{
  \maps[datatype=bibtex]{
    \map{
      \step[fieldset=issn, null]
      \step[fieldset=doi, null]
      \step[fieldset=url, null]
    }
    \map[overwrite=true]{
      \step[fieldsource=fjournal]
      \step[fieldset=journal, origfieldval]
    }
  }
}
\usepackage{graphicx}

\definecolor{softred}{rgb}{.933,.4,.46.7}
\definecolor{limegreen}{rgb}{.133,.533,.2}
\definecolor{moderateblue}{rgb}{.267,.467,.667}

\usepackage{xparse}

\usepackage{tikz}
\usepackage{tikz-3dplot}
\usepackage{pgfplots}
\pgfplotsset{compat=newest}
\usepgfplotslibrary{fillbetween}
\usetikzlibrary{shapes.misc}
\usetikzlibrary{positioning,calc,intersections}
\usetikzlibrary{through, pgfplots.fillbetween, arrows.meta}
\pgfdeclarelayer{bg}
\pgfsetlayers{bg,main}

\makeatletter
\tikzset{use path/.code=\tikz@addmode{\pgfsyssoftpath@setcurrentpath#1}}
\makeatother

\NewDocumentCommand{\bisettrice}{%
	O{}     
	mmm     
	m       
	O{1}O{1}
}{%
	\path[name path=Bis#2#3#4] let
	\p1 = ($(#2) - (#3)$),
	\p2 = ($(#4) - (#3)$),
	\n1 = {veclen(\x1,\y1)} ,
	\n2 = {veclen(\x2,\y2)} ,
	\n3 = {max(\n1,\n2)},
	\p1 = ($(#3)!\n3!(#2)$),
	\p2 = ($(#3)!\n3!(#4)$),
	\p3 = ($(\p1) + (\p2) - (#3)$)
	in
	(#3) -- (\p3) ;
	
	\path[name path = foo] (#2)--(#4) ;
	
	\path[name intersections={of=foo and Bis#2#3#4, by=#5}] ;
	
	\path[#1] ($(#3)!#6!(#5)$) -- ($(#5)!#7!(#3)$) ;
}

\newcommand{\eps}{\epsilon}
\newcommand{\la}{\lambda}

\newcommand{\ga}{\gamma}

\newcommand{\ka}{\kappa}
\newcommand{\al}{\alpha}
\newcommand{\de}{\delta}

\newcommand{\R}{\mathbb{R}}

\newcommand{\I}{\mathbbm{1}}

\let\P\relax \newcommand{\P}{\mathbb{P}}
\newcommand{\E}{\mathbb{E}}

\newcommand{\rr}{\mathcal{R}}
\newcommand{\xx}{\mathcal{X}}
\newcommand{\yy}{\mathcal{Y}}
\newcommand{\zz}{\mathcal{Z}}

\newcommand{\hh}{\mathcal{H}}

\DeclareMathOperator{\sign}{sign}
\DeclareMathOperator*{\argmax}{arg\,max}
\DeclareMathOperator*{\argmin}{arg\,min}
\DeclareMathOperator{\J}{J\!}
\DeclareMathOperator{\co}{co}

\newcommand{\h}[1]{\widehat{#1}}

\theoremstyle{definition}
\newtheorem{dfn}{Definition}

\theoremstyle{remark}

\theoremstyle{plain}
\newtheorem{lem}[dfn]{Lemma}
\newtheorem{prop}[dfn]{Proposition}
\newtheorem*{prop*}{Proposition}
\newtheorem{thm}[dfn]{Theorem}
\newtheorem*{thm*}{Theorem}

\title{\bf Multiclass learning with margin: \\
exponential rates with no bias-variance trade-off}

\date{}

\makeatletter
\newcommand{\printfnsymbol}[1]{%
  \textsuperscript{\@fnsymbol{#1}}%
}
\makeatother

\author{%
 {\bf Stefano Vigogna} \hfill \hspace{15em}{\small\texttt{vigogna@dibris.unige.it}}\\
 {\small \it MaLGa - DIBRIS, University of Genova, Italy \hfill \hspace{12em}}\\ 
  \and
  {\bf Giacomo Meanti} \hfill \hspace{13.2em}{\small\texttt{giacomo.meanti@edu.unige.it}}\\
  {\small \it MaLGa - DIBRIS, University of Genova, Italy \hfill \hspace{12em}}\\   
  \and
  {\bf Ernesto De Vito} \hfill \hspace{15.4em}{\small\texttt{ernesto.devito@unige.it}}\\
  {\small \it MaLGa - DIMA, University of Genova, Italy \hfill \hspace{12em}}\\   
  \and
  {\bf Lorenzo Rosasco} \hfill \hspace{14.5em}{\small\texttt{lorenzo.rosasco@unige.it}}\\
  {\small \it MaLGa - DIBRIS, University of Genova, Italy  \hfill \hspace{12em}}\\
  {\small \it Istituto Italiano di Tecnologia, Genova, Italy \hfill \hspace{12em}}\\
  {\small \it CBMM - MIT, Cambridge, MA, USA \hfill \hspace{12em}}
}

\begin{document}

\maketitle

\begin{abstract}
We study the behavior of error bounds for multiclass classification under suitable margin conditions. For a wide variety of methods we prove that the classification error under a hard-margin condition decreases exponentially fast without any bias-variance trade-off. Different convergence rates can be obtained in correspondence of different margin assumptions. With a self-contained and instructive analysis we are able to generalize known results from the binary to the multiclass setting. 
\end{abstract}

\section{Introduction}
\label{sec:intro}

It was recently remarked that the learning curves observed in practice can be quite different from those predicted in theory \cite{zhang2021understanding}.
In particular, while one might expect performance to degrade as models get larger or less constrained \cite{hastie2009}, this is in fact not the case.
By the no free lunch theorem \cite{wolpert1996lack},
theoretical results critically depend on the set of assumptions made on the problem.
Such assumptions can be hard to verify in practice,
hence a possible way to tackle the seeming contradictions in learning theory \emph{vs.}~practice is to consider a wider range of assumptions,
and check whether the corresponding results can explain empirical observations.

In the context of classification, it is interesting to consider assumptions describing the difficulty of the problem in terms of {\em  margin}~\cite{Tsybakov1999,Tsybakov2004}.
It is well known that very different learning curves can be obtained depending on the considered margin conditions \cite{bartlett2006convexity}. Further,
the behavior of the test error in terms of misclassification can be considerably different from that induced by the surrogate loss function used for empirical risk minimization \cite{Zhang2004,bartlett2006convexity}.
An extreme case is when there is a hard margin among the classes.  Indeed, in this case the misclassification error can decrease \emph{exponentially} fast as the number of points increases, 
while the surrogate loss error displays a polynomial decay. This behavior was first noted in \cite{Koltchinskii2005, Tsybakov2007} for a wide class of estimators (see also \cite{yao07}), 
and reprised more recently in \cite{vivien18, nitanda2019} for stochastic gradient descent. 
The effect of margin conditions has also been considered for \emph{multiclass} learning \cite{zhang04b, sun06, NIPS2012_1cecc7a7},
but not in the hard-margin case.
Interestingly, hard-margin and exponential rates have been studied by \cite{cabannes21}
in the context of structured prediction \cite{nowak19}.
However, these latter results are restricted to least-squares-based estimators.

The purpose of our paper is twofold. On the one hand, we analyze the effect of margin conditions, and in particular hard-margin conditions, for a wide class of multiclass estimators derived from different surrogate losses.
On the other hand, we build on ideas  in \cite{NIPS2012_1cecc7a7, vivien18, nitanda2019} to provide a simplified and self-contained treatment that naturally recovers results for binary classification as a special case. 
In particular, we note  that, in the presence of a hard margin,  the misclassification error curve does not exhibit any \emph{bias-variance} trade-off, thus providing a possible explanation to the empirical observations that motivate our study.

The rest of the paper is organized as follows.
We conclude the introduction by setting up some basic notation.
In \Cref{sec:setting} we describe the multiclass classification problem, the surrogate approach and the simplex encoding. In \Cref{sec:analysis} we analyze the bias-variance decomposition for the misclassification risk, discuss soft and hard-margin conditions, and prove our main results of exponential convergence under assumptions of hard margin. In \Cref{sec:experiments} we validate the theory with experiments on synthetic data. Some final remarks are provided in \Cref{sec:conclusions}.

\paragraph{Notation.}

We will be using the following general notation.
$ a \lesssim b $ means that $ a \le c b $ for some positive absolute constant $c$.
The Euclidean norm and inner product of vectors $ w,w' \in \R^p $ are denoted by $\|w\|$ and $ \langle w , w' \rangle $, respectively.
For an event $E$, $ \I \{E\} $ denotes its indicator function,
and $\P\{E\}$ its probability.
The expectation of a random variable $Z$ is denoted by $\E Z$;
when the expectation is taken only with respect to a random variable $X$
(but $Z$ possibly depends also on other variables), we write $ \E_X Z $.
Conditioning of events or random variables on an event $E$ is indicated by $ \cdot \mid E $.
$ L^0(\xx,\zz) $ is the space of measurable functions on the (probability) measure space $\xx$ and with values in $ \zz\subset\R^p $,
and $ L^\infty(\xx,\zz) $
the subspace of essentially bounded functions, with norm $ \| \cdot \|_\infty $.

\section{Setting}
\label{sec:setting}

We consider a standard multiclass learning problem.
Let $ (X,Y) \in \xx \times \yy $ be a random pair,
where $ \xx \subset \R^d $
and $ \yy $ is a finite set of $ T \ge 2 $ elements.
We call the elements of $\yy$ \emph{classes},
and a (measurable) function $ c : \xx \to \yy $ a \emph{classifier}.
The \emph{misclassification risk} of a classifier $c$ is
$$
 \rr(c) = \P \{ c(X) \ne Y \} .
$$
Let
$$
 \rho( y \mid x ) = \P \{ Y = y \mid X = x \}
$$
denote the conditional probability of the class $y$ given the observation $x$.
The risk $\rr$ is minimized by the \emph{Bayes rule}
$$
 c_* (x) = \argmax_{y\in\yy} \rho(y \mid x) .
$$
We denote the minimum risk by $ \rr_* = \rr(c_*) $.
Given $n$ independent copies $ (X_i,Y_i) $ of $ (X,Y) $, $ i = 1 , \dots , n $,
the goal is to learn a classifier $ \widehat{c}$
such that $ \rr(\widehat{c}) - \rr_* \to 0 $ in expectation as $ n \to \infty $.
More precisely, we are interested in finite-sample bounds of the form
$$
 \E \rr(\widehat{c}) - \rr_* \lesssim a_n ,
$$
where $ a_n \to 0 $ gives a rate of convergence.

Empirical risk minimization would prescribe to compute $\widehat{c}$
by minimizing a sample version of $\rr$.
The misclassification risk can be seen as the expectation of the \emph{0-1 loss}
$$
  \I \{ y \ne y' \} , \qquad y , y' \in \yy .
$$
The empirical mean would thus be $ \frac{1}{n} \sum_{i=1}^n \I \{ c(X_i) \ne Y_i \} $.
However, the 0-1 loss is neither smooth nor convex,
and optimizing it is in general an NP-hard combinatorial problem \cite{feldman2012}.
A viable strategy is to replace the 0-1 loss
with a convex \emph{surrogate},
and the space of classifiers with a suitable linear space of vector-valued functions.
To do this, it is necessary to choose a vector encoding of the classes $ \yy \hookrightarrow \R^p $
and a decoding operator $ D : \R^p \to \yy $.
Following \cite{NIPS2012_1cecc7a7}, we encode the $T$ classes as the vertices of a $(T-1)$-simplex embedded in $\R^{T-1}$ (see \Cref{fig:simplex}).
\begin{figure}[h]
	\centering
	\begin{tikzpicture}
		\begin{axis}[
			height=.25\linewidth,
			axis lines=middle,
			y axis line style={draw=none},
			xmin=-1.4,
			xmax=1.4,
			ymin=-.5,
			ymax=.5,
			ticks=none,]
			\node[circle,fill,inner sep=2pt] (p1) at (axis cs:1,0) {};
			\node[circle,fill,inner sep=2pt] (p2) at (axis cs:-1,0) {};
			\draw[black, line width=1.5pt, line cap=round] (p1) -- (p2);
		\end{axis}
	\end{tikzpicture}
	\hspace{30pt}
\begin{tikzpicture}
\begin{axis}[
	height=.25\linewidth,
	axis lines=middle,
	axis equal,
	xmin=-0.5,
	xmax=1.2,
	ymin=-1,
	ymax=1,
	ticks=none,]
	\node[circle,fill,inner sep=2pt] (p1) at (axis cs:1,0) {};
	\node[circle,fill,inner sep=2pt] (p2) at (axis cs:-0.5,0.8660254) {};
	\node[circle,fill,inner sep=2pt] (p3) at (axis cs:-0.5,-0.8660254) {};
	\addplot[patch, color=blue!25, fill opacity = 0.25, faceted color=black, line width=0.5pt] coordinates{(1,0) (-0.5,0.8660254) (-0.5,-0.8660254)};
\end{axis}
\end{tikzpicture}
\hspace{30pt}
	\begin{tikzpicture}
	\begin{axis}[height=.45\linewidth,
		axis lines=middle,
		axis equal,
		xmin=-1,
		xmax=1,
		ymin=-1,
		ymax=1,
		zmin=-1,
		zmax=1,
		ticks=none,
		view={135}{30}
		]
		
		\def\opacity{0.7}
		\def\fillcol{blue!25}
		\coordinate (c1) at ( 1.        ,  0.        ,  0.        );
		\coordinate (c2) at (-0.33333333,  0.94280904,  0.        );
		\coordinate (c3) at (-0.33333333, -0.47140452,  0.81649658);
		\coordinate (c4) at (-0.33333333, -0.47140452, -0.81649658);
		\addplot3[patch, color=\fillcol, fill opacity=\opacity, faceted color=black, line width=0.5pt] coordinates{
			( 1.        ,  0.        ,  0.        )
			(-0.33333333,  0.94280904,  0.        )
			(-0.33333333, -0.47140452,  0.81649658)
		};
		\addplot3[patch, color=\fillcol, fill opacity=\opacity, faceted color=black, line width=0.5pt] coordinates{
			(-0.33333333, -0.47140452, -0.81649658)
			( 1.        ,  0.        ,  0.        )
			(-0.33333333,  0.94280904,  0.        )
		};
		\addplot3[patch, color=\fillcol, fill opacity=\opacity, faceted color=black, line width=0.5pt] coordinates{
			(-0.33333333, -0.47140452,  0.81649658)
			(-0.33333333, -0.47140452, -0.81649658)
			( 1.        ,  0.        ,  0.        )
		};
		\addplot3[patch, color=\fillcol, fill opacity=\opacity, faceted color=black, line width=0.5pt] coordinates{
			(-0.33333333,  0.94280904,  0.        )
			(-0.33333333, -0.47140452,  0.81649658)
			(-0.33333333, -0.47140452, -0.81649658)
		};
		\node[circle,fill,inner sep=2pt] (p1) at (c1) {};
		\node[circle,fill,inner sep=2pt] (p2) at (c2) {};
		\node[circle,fill,inner sep=2pt] (p3) at (c3) {};
		\node[circle,fill,inner sep=2pt] (p4) at (c4) {};
	\end{axis}
\end{tikzpicture}
\caption{Simplex encoding for $ T = 2,3,4 $.}
\label{fig:simplex}
\end{figure}
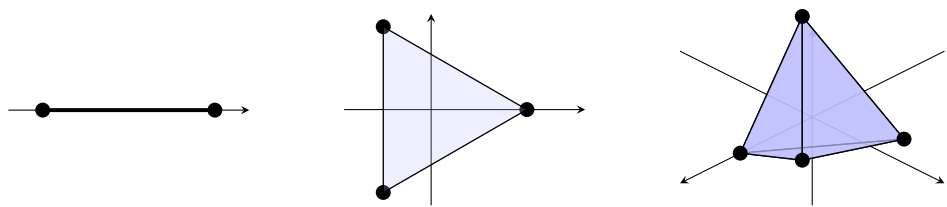%
For notational convenience we denote $\yy$ itself with its simplex encoding, 
that is $\yy$ is the set of points in 
in $\R^{T-1}$ such that
$$
 \|y\| = 1, \qquad \langle y , y' \rangle = -\tfrac{1}{T-1} .
$$
The decoding operator assigns a vector to the class with largest projection,
with ties arbitrarily broken (see \Cref{fig:decoder}):
$$
 D : \R^{T-1} \to \yy , \qquad D(w) = \argmax_{y\in\yy} \langle w, y \rangle .
$$
\begin{figure}[H]
\centering
\begin{tikzpicture}[scale=1.5,baseline]
	\tikzset{cross/.style={cross out, draw=black, minimum size=2*(#1-\pgflinewidth), inner sep=0pt, outer sep=0pt},
             cross/.default={3pt}}
    \tikzset{sline/.style={line width=1.3pt}}
    \tikzset{dline/.style={dashed, line width=0.5pt}}
    \tikzset{simplv/.style={circle, fill, inner sep=2pt}}
    \tikzset{arrow/.style={-{latex[length=1mm]}, line width=0.3pt}}
    
	\coordinate (c1) at (1, 0);
	\coordinate (c2) at (-1,0);
	\coordinate (o) (0, 0);
	\draw[sline] (o) -- (c1);
	\draw[sline] (o) -- (c2);
	\node[simplv,softred] (n1) at (c1) {};
	\node[simplv,limegreen] (n2) at (c2) {};
			\draw[name path=bis, dline] (0, -0.8660254) -- (0, 0.8660254);
	\coordinate (w1) at (0.3, 0);
	\node[above right=0pt of {w1}, outer sep=0pt] {$w$};
	\draw (w1) node[simplv,softred] {};
	\node[right=0pt of {c1}] {$D(w)$};
	\draw[arrow] (w1) -- (c1);
	\coordinate (w2) at (-0.6, 0);
	\draw (w2) node[simplv,limegreen] {};
	\draw[arrow] (w2) -- (c2);
\end{tikzpicture}
\hspace{50pt}
\begin{tikzpicture}[scale=1.5,baseline]
	\tikzset{cross/.style={cross out, draw=black, minimum size=2*(#1-\pgflinewidth), inner sep=0pt, outer sep=0pt},
             cross/.default={3pt}}
    \tikzset{sline/.style={line width=1.3pt}}
    \tikzset{dline/.style={dashed, line width=0.5pt}}
    \tikzset{simplv/.style={circle, fill, inner sep=2pt}}
    \tikzset{arrow/.style={-{latex[length=1mm]}, line width=0.3pt}}
    
	\coordinate (c1) at (1, 0);
	\coordinate (c2) at (-0.5,0.8660254);
	\coordinate (c3) at (-0.5,-0.8660254);
	\coordinate (o) (0, 0);
	\draw[sline] (o) -- (c1);
	\draw[sline] (o) -- (c2);
	\draw[sline] (o) -- (c3);
	\node[simplv,softred] (n1) at (c1) {};
	\node[simplv,limegreen] (n2) at (c2) {};
	\node[simplv,moderateblue] (n3) at (c3) {};
	\bisettrice[draw, dline]  {c1}{o}{c2}{b1}[2];
	\bisettrice[draw, dline]  {c2}{o}{c3}{b2}[2];
	\bisettrice[draw, dline]  {c3}{o}{c1}{b3}[2];
	\coordinate (w1) at (0.8, 0.6);
	\node[above right=0pt of {w1}, outer sep=0pt] {$w$};
	\draw (w1) node[simplv,softred] {};
	\node[right=0pt of {c1}] {$D(w)$};
	\draw[arrow] (w1) -- (c1);
	\coordinate (w2) at (-0.4, 0.35);
	\draw (w2) node[simplv,limegreen] {};
	\draw[arrow] (w2) -- (c2);
	\coordinate (w3) at (0.1, -0.45);
	\draw (w3) node[simplv,moderateblue] {};
	\draw[arrow] (w3) -- (c3);
\end{tikzpicture}
\caption{Simplex decoding ($ T = 2,3 $).}
\label{fig:decoder}
\end{figure}
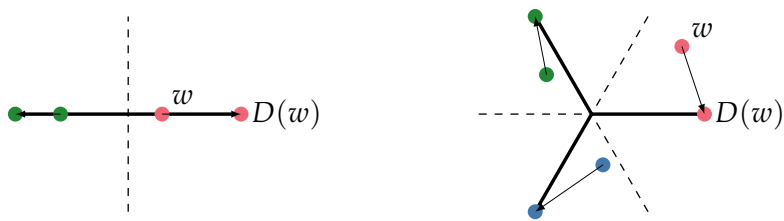
In the case of binary classification ($T=2$),
we have $ \yy = \{ \pm1 \} \subset \R $ and $ D(w) = \sign(w) $.
A \emph{plug-in} classifier $ Df(x) = D(f(x)) $ can be defined
by composing a vector-valued function $ f : \xx \to \R^{T-1} $ with the decoding operator.
The simplex coding offers some advantages over other common types of coding, such as one-hot.
First, as we just saw, it is perfectly consistent with the standard $ (\{\pm 1\}, \sign) $ coding of binary classification.
Second, it automatically satisfies structural constraints
that other codings need to impose additionally on the hypothesis class;
as the so-called sum to zero constraint,
which makes both numerical implementation and theoretical analysis more involved.

To identify the \emph{target function} to plug into the decoder,
we fix a convex surrogate loss
$$
 \ell : \R^{T-1} \times \yy \to [0,\infty)
$$
with corresponding risk $ \rr_\ell (f)  =  \E \ell (f(X),Y) $, and define
$$
 f_\ell = \argmin_{f \in L^0(\xx,\R^{T-1}) }  \rr_\ell (f) .
$$
We then approximate $f_\ell$ by a (uniform) approximator $f_\la$.
At the current level of generality, $\la$ simply denotes a generic parameter to be tuned.
For instance, $f_\la$ can be the minimizer of a regularized risk,
with $\la$ the regularization parameter.
Finally, our classifier will be $ D\widehat{f}_\la $,
with $\widehat{f}_\la$ the empirical estimate of $f_\la$ based on the samples $ \{ (X_i,Y_i) \}_{i=1}^n $.

We are going to consider two cases of loss functions.
The first case is the square loss $ \ell(w,y) = \| w - y \|^2  $,
for which $ f_\ell (x) = \eta(x) $, where
$$
 \eta(x) = \E [ Y \mid X = x ]
$$
is the \emph{regression function}.
The second case is a family of functions of the \emph{margin} $ \langle w , y \rangle $,
namely losses of the form $ \ell_\phi(w,y) = \phi(\langle w , y \rangle) $ for a suitable (differentiable, convex) function $ \phi : \R \to [0,\infty) $.
Examples of $\phi$ are $ \phi(t) = \ln(2)^{-1} \ln ( 1 + e^{-t} ) $ and $ \phi (t) = e^{-t} $,
generalizing the logistic and exponential loss to the multiclass setting, respectively.
For margin losses, we will denote the minimizer $ f_\ell $ by $ f_\phi $.
Note that in binary classification the square loss is itself a function of the margin,
$ \ell(w,y) = \| 1 - wy \|^2 $,
while for $ T \ge 3 $ this is no longer the case.

\section{Analysis}
\label{sec:analysis}

We start by analyzing the peculiar structure of the bias-variance decomposition in classification.
As we will see, the key point is that the bias can be made zero
under suitable margin conditions.
When only the variance is left,
the misclassification error can be controlled by uniform concentration.
These general facts can then be applied to different loss functions,
leading to our main results.

\subsection{Bias-variance for plug-in classifiers} \label{sec:bias-variance}

To analyze the performance of a plug-in classifier $D\widehat{f}_\la$,
we decompose the excess misclassification risk as 
\begin{align}
 \rr( D\widehat{f}_\la ) - \rr_*  = \ & \rr( D\widehat{f}_\la ) - \rr( D f_\la ) \label{eq:variance} \\
 + \ & \rr( D f_\la ) - \rr( Df_\ell ) \label{eq:bias} \\
 + \ & \rr( Df_\ell ) - \rr_* . \label{eq:fisher}
\end{align}
The last term results from replacing the 0-1 loss with the surrogate loss $\ell$.
Loss functions for which $ Df_\ell = c_* $, and therefore \eqref{eq:fisher} is zero,
are called \emph{Fisher consistent} (or \emph{classification calibrated}).
Fisher consistency is a common and well characterized property~\cite{Zhang2004}.
In particular, the square loss is Fisher consistent (see \Cref{lem:margin}).
For margin losses, consistency will be assumed in all that follows,
and shown in some examples.

\begingroup
\setlength\intextsep{0pt}
\begin{wrapfigure}[8]{r}{.35\linewidth}
\centering
\begin{tikzpicture}[
	scale=1,
	dot/.style = {circle, fill, minimum size=#1,
		inner sep=0pt, outer sep=0pt},
	dot/.default = 4pt,  
	]
	\clip (-1,-1.75) rectangle (3.1, 1.75);
	\node[dot, label=right:$f_\ell$] (fl) at (2, -0.1) {};
	\node[dot, label=right:$f_\lambda$] (fla) at (-0.8, 0.5) {};
	\node[dot, label=above right:$f_{\lambda_*}$] (flas) at (1, 0.5) {};
	\node[dot, label=right:$\eta$] (eta) at (1, -1) {};
	\draw[-{Latex[length=1.5mm]},] (fla) to[out=45,in=135] (flas);
	\draw[-{Latex[length=1.5mm]},] (flas) to[out=-40,in=160] (fl);
	
	\draw (0,0) .. controls (1, 3) and (3.2, 1) .. (3, 0) .. controls (2.66, -2.5) and (-0.5, -2) .. (0, 0) -- cycle;
\end{tikzpicture}
\end{wrapfigure}
The term \eqref{eq:bias} is a bias term.
Crucially, it can be set to zero
for a wide range of parameters $\la$.
The idea is that we can have $ \rr(Df_\la) = \rr(Df_\ell) $
even when $ \rr_\ell(f_\la) \gg \rr_\ell(f_\ell) $.
Here is a fundamental difference between regression and classification.
While in regression $f_\ell$ is a target point,
in classification it is rather a representative of the target class
$$
 [ f_\ell ]  = \{ f \in L^0(\xx,\R^{T-1}) : Df = Df_\ell \text{ almost surely} \} .
$$
Hence, it is enough for $f_\la$ to land in $[ f_\ell ]$,
possibly far from $f_\ell$ itself.
This is easier if the class $[ f_\ell ]$ is ``large'',
which can be ensured by imposing special margin conditions.
Assuming that $\ell$ is Fisher consistent, a generic function $f$ lies in $[f_\ell]$ if and only if
\begin{equation} \label{eq:Df=c*}
 Df = c_* \quad \text{almost surely} .
\end{equation}
Chosen a Fisher consistent loss and put the bias to zero,
all that’s left is the variance term \eqref{eq:variance}:
\begin{equation} \label{eq:misc}
 \rr( D\widehat{f}_\la ) - \rr_*  = \rr( D\widehat{f}_\la ) - \rr( D f_\la ) .
\end{equation}
At this point, $\la$ is set and needs no trade-off.
Fast convergence of the variance, and therefore of the whole excess misclassification risk,
can be derived using once again margin conditions.
\endgroup

\subsection{Margin conditions} \label{sec:margin}

In binary classification,
the \emph{margin conditions},
also known as Tsybakov's low-noise assumptions
\cite{Tsybakov1999,Tsybakov2004,Koltchinskii2005,Tsybakov2007},
are a set of assumptions under which it is possible to obtain
fast convergence (up to exponential) for plug-in classifiers.
They can be stated as follows:
there exists $ \alpha \in (0,\infty] $ such that, for every $ \de > 0 $,
\begin{equation} \label{eq:margin2}
 \P \{ | \eta(X) | \le \de \} \lesssim \de^\al .
\end{equation}
In the extreme case of $ \al = \infty $, we get
\begin{equation} \label{eq:hard-margin2}
 |\eta(X)| \ge \de \quad \text{almost surely} ,
\end{equation}
which is sometimes referred to as the \emph{hard-margin} condition.

Following~\cite{NIPS2012_1cecc7a7,nowak19}, we can generalize \eqref{eq:margin2} and \eqref{eq:hard-margin2} to the multiclass setting. 
For $ w \in \R^{T-1} $, we define the \emph{decision margin}
$$
 M(w) = \min_{ y \ne D(w) } \langle w , D(w) - y \rangle .
$$
$M(w)$ is the difference between the largest and the second largest projection of $w$ onto $\yy$,
namely the confidence gap between first and second guess.
For $ T = 2 $, we have $ M(w) = 2 |w| $.
In general, we say that a function $ f : \xx \to \R^{T-1} $ satisfies the margin condition with exponent $ \al \in (0,\infty] $ if, for every $ \de > 0 $,
\begin{equation} \label{eq:marginf}
 \P \{ M(f(X)) \le \de \} \lesssim \de^\al .
\end{equation}
In particular, one can take $ f = \eta $, which for $T=2$ gives back \eqref{eq:margin2}.
Again, $ \al = \infty $ gives the hard-margin condition (see \Cref{fig:hard-margin})
\begin{equation} \label{eq:hard-marginf}
 M(f(X)) \ge \de \quad \text{almost surely} .
\end{equation}
\begin{figure}[h]
\centering
\begin{tikzpicture}[scale=1.8,baseline]
	\tikzset{cross/.style={cross out, draw=black, minimum size=2*(#1-\pgflinewidth), inner sep=0pt, outer sep=0pt},
		cross/.default={3pt}}
	\tikzset{sline/.style={line width=1pt}}
	\tikzset{dline/.style={dashed, line width=0.5pt}}
	\tikzset{simplv/.style={circle, fill, inner sep=1.6pt}}
	\tikzset{extended/.style={shorten >= -3cm, shorten <= 0}}
	
	\def\inflate{0.3}
	\def\height{0.1}
	\begin{scope}
		\clip(-1.0,-\height) rectangle (1.0, \height);
		
		\coordinate (c1) at (0.7, 0);
		\coordinate (c2) at (-0.7, 0);
		\coordinate (o) (0, 0);
		\fill [limegreen!75] (-1.3,-\height) rectangle ($(0,\height) - (0.3, 0)$);
		\fill [softred!75] ($(0,-\height) + (0.3, 0)$) rectangle (1.3, \height);
		\node[simplv] (n1) at (c1) {};
		\node[simplv] (n2) at (c2) {};
		\draw[name path=bis, dline] (0, -\height) -- (0, \height);
		\draw[] ($(0, -\height) - (\inflate, 0)$) -- ($(0, \height) - (\inflate, 0)$);
		\draw[] ($(0, -\height) + (\inflate, 0)$) -- ($(0, \height) + (\inflate, 0)$);
		
		\draw[{LaTeX[width=1.5mm, length=0.9mm]}-{LaTeX[width=1.5mm, length=0.9mm]}] ($(0,-0.0) - (\inflate,0)$) -- ($(0, -0.0) + (\inflate,0)$) node[midway, fill=white] {$\delta$};
	\end{scope}
\end{tikzpicture}\hspace{50pt}
\begin{tikzpicture}[scale=1.2,baseline]
	\tikzset{cross/.style={cross out, draw=black, minimum size=2*(#1-\pgflinewidth), inner sep=0pt, outer sep=0pt},
		cross/.default={3pt}}
	\tikzset{sline/.style={line width=1pt}}
	\tikzset{dline/.style={dashed, line width=0.5pt}}
	\tikzset{simplv/.style={circle, fill, inner sep=1.6pt}}
	\tikzset{extended/.style={shorten >= -3cm, shorten <= 0}}
	
	\def\inflate{0.4}
	\def\figxl{-1.2}
	\def\figyl{-1.2}
	\def\figxr{1.5}
	\def\figyr{1.2}
	\begin{scope}
		\clip(\figxl,\figyl) rectangle (\figxr, \figyr);
		
		\coordinate (c1) at (1, 0);
		\coordinate (c2) at (-0.5,0.8660254);
		\coordinate (c3) at (-0.5,-0.8660254);
		\coordinate (o) (0, 0);
		\node[simplv] (n1) at (c1) {};
		\node[simplv] (n2) at (c2) {};
		\node[simplv] (n3) at (c3) {};
		\bisettrice[draw, dline]  {c1}{o}{c2}{b1}[5];
		\bisettrice[draw, dline]  {c2}{o}{c3}{b2}[5];
		\bisettrice[draw, dline]  {c3}{o}{c1}{b3}[5];
		\coordinate (b1d) at ($5*(b1)$);
		\coordinate (b2d) at ($5*(b2)$);
		\coordinate (b3d) at ($5*(b3)$);
		
		\coordinate (infl1) at ($(o)!\inflate!(c1)$);	
		\draw[extended, name path=infl1a] (infl1) -- +($(b1d)-(o)$);
		\draw[extended, name path=infl1b] (infl1) -- +($(b3d)-(o)$);
		\coordinate (infl2) at ($(o)!\inflate!(c2)$);
		\draw[extended, name path=infl2a] (infl2) -- +($(b1d)-(o)$);
		\draw[name path=infl2b] (infl2) -- +($(b2d)-(o)$);
		\coordinate (infl3) at ($(o)!\inflate!(c3)$);
		\draw[extended, name path=infl3a] (infl3) -- +($(b2d)-(o)$);
		\draw[extended, name path=infl3b] (infl3) -- +($(b3d)-(o)$);
		
		\path[name path=horiz] (-1, 0.95) -- (2, 0.95);
		\path[name intersections={of=infl2a and horiz, by=t-start}];
		\path[name intersections={of=infl1a and horiz, by=t-end-par}];
		\coordinate (t-end-par) at ($(infl1)!(t-start)!(t-end-par)$);
		\draw[{LaTeX[width=1.5mm, length=0.9mm]}-{LaTeX[width=1.5mm, length=0.9mm]}] (t-start)--(t-end-par) node[midway, fill=white] {$\delta$};
		
		\path[name path=clip-path] (\figxl, \figyl) -- (\figxl, \figyr) -- 
								   (\figxr, \figyr) -- (\figxr, \figyl) -- (\figxl, \figyl);
		\begin{pgfonlayer}{bg}
			\clip(\figxl,\figyl) rectangle (\figxr, \figyr);
			\path[
			intersection segments={
				of=infl2a and infl2b,
				sequence={L2--R2}
			},name path=AB];
			\path[
			fill=limegreen!75,
			intersection segments={
				of=AB and clip-path,
				sequence={L2--R2}
			}];
			\path[
				intersection segments={
				of=infl1a and infl1b,
				sequence={L2--R2}
			},name path=AB];
			\path[
				fill=softred!75,
				intersection segments={
				of=AB and clip-path,
				sequence={L2--R2}
			}];
			\path[
				intersection segments={
				of=infl3a and infl3b,
				sequence={L*--R*}
			},name path=AB];
			\path[
				fill=moderateblue!75,
				intersection segments={
				of=clip-path and AB,
				sequence={L3--L1--R*}
			}];
		\end{pgfonlayer}
	\end{scope}
\end{tikzpicture}

\caption{Hard-margin condition ($T=2,3$).}
\label{fig:hard-margin}
\end{figure}
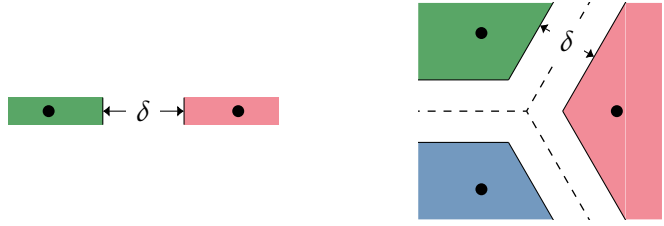
Intuitively, these conditions say that the probability of falling in a ``runoff zone'',
where the plugging-in function would be ``uncertain'',
is either (polynomially) small \eqref{eq:marginf}, or zero \eqref{eq:hard-marginf}.
The reason why we state \eqref{eq:marginf} and \eqref{eq:hard-marginf} for an arbitrary $f$ is that
we will transfer these properties to minimizers $ f_\ell $ (and $f_\la$) of general (regularized) losses,
including but not limited to the square.
Combining Fisher consistency and hard margin, we obtain the following stronger condition.
\begin{lem} \label{lem:D-margin}
A function $ f \in L^0(\xx,\R^{T-1}) $ satisfies \eqref{eq:Df=c*} and \eqref{eq:hard-marginf}
if and only if
\begin{equation} \label{eq:D-margin}
 \min_{y \ne c_*(X)} \langle f(X) , c_*(X) - y \rangle \ge \de \quad \text{\textnormal{almost surely}} .
\end{equation}
\end{lem}
\begin{proof}
First note that, if \eqref{eq:Df=c*} holds, then \eqref{eq:D-margin} is the same as \eqref{eq:hard-marginf}.
Now suppose \eqref{eq:D-margin} holds. Then
$$
 \langle f(X) , c_*(X) \rangle > \max_{y \ne c_*(X)} \langle f(X) , y \rangle ,
 $$
hence $ Df(X) = c_*(X) $, that is, \eqref{eq:Df=c*} holds too.
\end{proof}

Beside \eqref{eq:marginf} and \eqref{eq:hard-marginf}, we will also consider another generalization of \eqref{eq:hard-margin2}
which is independent of any particular classifier,
and instead is stated purely in terms of the conditional probabilities.
To illustrate such a condition, we note that \eqref{eq:hard-margin2} is equivalent to saying that
either $ \rho( 1 | x ) $ or $ \rho(-1 | x) $ is no less than $1/2 + \de/2$
(for almost every $x$, there is one class with probability bounded away from coin flipping).
This in turn is equivalent to
\begin{equation} \label{eq:hard-margin}
 \min_{y \ne c_*(X)} \rho( c_*(X) \mid X )  - \rho( y \mid X ) \ge \de \quad \text{almost surely} ,
\end{equation}
which says that the most probable class has almost always an edge of $\de$
over the second most probable class.
Since this inequality makes sense for arbitrary $T$,
we take it as our hard-margin condition for multiclass problems.
More generally, one may consider problems where for some $\al \in (0,\infty]$
and all $\de>0$,
\begin{equation} \label{eq:margin}
 \P \{ \min_{y \ne c_*(X)} \rho( c_*(X) \mid X )  - \rho( y \mid X ) \le \de \} \lesssim \de^\al ,
\end{equation}
generalizing \eqref{eq:margin2} to $ T \ge 2 $.

The margin conditions on the conditional probabilities
can be related to those expressed on classifiers.
\begin{lem} \label{lem:margin}
 We have $ D\eta = c_* $ almost surely.
 Moreover, \eqref{eq:hard-margin} holds
 if and only if
 $\eta$ satisfies \eqref{eq:hard-marginf}.
\end{lem}
\begin{proof}
Let
$$
 \Delta = \{ p \in \R^{\yy} : p_y \ge 0 , \sum_{y\in\yy} p_y = 1 \}
$$
be the probability simplex on $\yy$,
and let $ \co \yy $ be the encoding simplex defined as the convex hull of $\yy$.
Then $\Delta$ and $\co\yy$ are canonically isomorphic via the barycenter coordinate map
$$
 \beta: \Delta \to \co\yy , \qquad \beta(p) = \sum_{y\in\yy} p_y y .
$$
Now consider the map
$$
 \rho : \xx \to \Delta , \qquad \rho(x) = [\rho(y \mid x)]_{y\in\yy} .
$$
Then we have $ \beta \circ \rho = \eta $.
It follows that $ y \in \yy $ maximizes $ \rho(y \mid x) $ if and only if it maximizes $ \langle \eta(x) , y \rangle  $.
Therefore, $ D\eta=c_* $.
The same holds for maximizing over $ y \ne c_*(x) $,
whence the second claim.
\end{proof}

\subsection{Misclassification comparison} \label{sec:comparison}

In view of \eqref{eq:misc},
we need in fact to compare the misclassification risk of two classifiers.
This can be done by introducing a bounding distance.
Since the distance will be symmetric,
the resulting bound gives a symmetric comparison between any two classifiers,
as opposed to the usual comparison of a classifier with respect to a fixed (Bayes) rule.
For this reason, the following results may be of independent interest.

We define the \emph{Hamming distance} of $ c' , c \in L^0(\xx,\yy) $ as
$$
 r( c' , c ) = \P \{ c'(X) \ne c(X) \} .
$$
The Hamming distance bounds the difference of misclassification risk.
\begin{lem} \label{lem:R<r}
For every $ c' , c \in L^0(\xx,\yy) $,
$$
  | \rr(c') - \rr(c) | \le r(c',c) .
$$
\end{lem}
\begin{proof}
By direct computation,
 \begin{align*}
  | \rr (c') - \rr(c) |
  & = | \E [ \I \{ c'(X) \ne Y \} - \I \{ c(X) \ne Y \} ] | \\
  & \le \E [ | \I \{ c'(X) \ne Y \} - \I \{ c(X) \ne Y \} | ] \\
  & \le \E [ \I \{ c'(X) \ne c(X) \} ] = r ( c' , c ) . \qedhere
 \end{align*}
\end{proof}
The next step is to bound the Hamming distance between two plug-in classifiers.
\begin{lem} \label{lem:r<infty}
 For every $ f', f \in L^\infty(\xx,\R^{T-1}) $,
 $$
  r(Df',Df) \le \P \left\{ \| f' - f \|_\infty \ge \sqrt{\tfrac{T-1}{2T}} M(f(X)) \right\} .
 $$
\end{lem}
\begin{proof}
Let $ Df'(x) = y' \ne y = Df(x) $.
Then
\begin{align*}
 \min_{j \ne y'} \langle y' - j , f'(x) \rangle
 & = \langle y' , f'(x) \rangle - \max_{j \ne y'} \langle j , f'(x) \rangle \\
 & \le \langle y' , f'(x) \rangle - \langle y , f'(x) \rangle \\
 & \le \langle y' - y , f'(x) \rangle - \langle y' - y , f(x) \rangle \\
 & = \langle y' - y , f'(x) - f(x) \rangle \\
 & \le \| y' - y \| \| f'(x) - f(x) \| \\
 & \le \sqrt{\tfrac{2T}{T-1}} \| f' - f \|_\infty . \qedhere
\end{align*}
\end{proof}
Now we let the samples come into play.
Let $ \widehat{f} \in L^\infty(\xx,\R^{T-1}) $ be a function of $(X_i,Y_i)$, $ i = 1,\dots,n $,
such that, for every $ \eps > 0 $ and some constant $ b > 0 $,
\begin{equation} \label{eq:inf<exp}
 \P \{ \| \h{f} - f \|_\infty > \eps \} \lesssim \exp(-n \eps^2 / b^2) .
\end{equation}
Then the following polynomial and exponential bounds hold true.
\begin{prop} \label{prop:poly-exp}
Suppose $f$ satisfies the margin condition \eqref{eq:marginf},
and let $\widehat{f}$ obey the concentration \eqref{eq:inf<exp}.
Then
 $$
  \E | \rr(D\widehat{f}) - \rr(Df) | \lesssim b^\al \left(\tfrac{2T}{T-1}\right)^{\al/2} \left(\tfrac{\log n^{\al/2}}{n}\right)^{\al/2} .
 $$
 If $f$ satisfies the hard-margin condition \eqref{eq:hard-marginf}, then
 $$
  \E | \rr(D\widehat{f}) - \rr(Df) | \lesssim \exp(-n \de^2 / b^2) .
 $$
\end{prop}
\begin{proof}
By \Cref{lem:R<r} and \Cref{lem:r<infty},
\begin{align*}
 & \E | \rr(D\widehat{f}) - \rr(Df) | \le \E r ( D \h{f} , D f ) \le \E_{{\{(X_i,Y_i)\}}_{i=1}^n} \E_X \I \left\{ \| \h{f} - f \|_\infty \ge \sqrt{\tfrac{T-1}{2T}} M(f(X)) \right\} .
\end{align*}
Let $ \ga = \sqrt{\tfrac{T-1}{2T}} M(f(X)) $ and $ E = \{ M(f(X)) \le \de \} $.
Then we have
\begin{align*}
\E_X \I \{ \| \h{f} - f \|_\infty \ge \ga \} & = \E_X [ \I \{ \| \h{f} - f \|_\infty \ge \ga \} \mid E \ ] \ \P\{E\} \\
 & + \E_X [ \I \{ \| \h{f} - f \|_\infty \ge \ga \} \mid E^\complement \ ] \ \P\{E^\complement\} \\
 & \le \P\{E\} + \I \{ \| \h{f} - f \|_\infty \ge \sqrt{\tfrac{T-1}{2T}} \de \} ,
\end{align*}
where $ \P\{E\} \lesssim \de^\al $ by \eqref{eq:marginf}.
Moreover, thanks to \eqref{eq:inf<exp},
\begin{align*}
\E_{{\{(X_i,Y_i)\}}_{i=1}^n} \I \{ \| \h{f} - f \|_\infty \ge \sqrt{\tfrac{T-1}{2T}} \de \}
 = \P \left\{ \| \h{f} - f \|_\infty \ge \sqrt{\tfrac{T-1}{2T}} \de \right\}
 \lesssim \exp(-n \tfrac{T-1}{2T} \de^2 / b^2) .
\end{align*}
Setting $ \de^2 = b^2 \tfrac{2T}{T-1} (\log (n^{\al/2})/n) $, we obtain the first claimed inequality.
The second inequality follows similarly using \eqref{eq:hard-marginf} in place of \eqref{eq:marginf}.
\end{proof}

\subsection{Main results}

In this section we establish exponential convergence of plug-in classifiers
under assumptions of hard margin.
We assume the setting of \Cref{sec:setting},
and use the arguments of \Cref{sec:bias-variance,sec:margin,sec:comparison}.
The main results are given for two cases of loss functions,
first for the square loss (namely, for the regression function),
and then for a general family of margin losses.
We will also be making the additional assumptions below.

\begin{enumerate}[label=\textnormal{(\roman*)}]
 \item $ f_\ell \in L^\infty(\xx,\R^{T-1}) $ and $ \| f_\la - f_\ell \|_\infty \xrightarrow[\la]{} 0 $ \label{it:fla-fl} .
\end{enumerate}
Further, let $ \h{f}_\la \in L^\infty(\xx,\R^{T-1}) $ be an estimate of $f_\la$,
and assume that, for every $\la$ and some $ b > 0 $, the following concentration bound holds true:
\begin{enumerate}[resume,label=\textnormal{(\roman*)}]
 \item $ \P \{ \| \h{f}_\la - f_\la \|_\infty > \eps \} \lesssim \exp(-n \eps^2 / b^2) $. \label{it:hfla-fla}
\end{enumerate}
Regularization methods in reproducing kernel Hilbert spaces (RKHS) \cite{scholkopf2002learning,steinwart2008support}
provide one framework where the properties \ref{it:fla-fl}, \ref{it:hfla-fla} can be satisfied.
In particular, one can fix a separable RKHS $ \hh \subset L^0(\xx,\R^{T-1}) $ with norm $ \| \cdot \|_\hh $,
and define
$$
 f_\la = \argmin_{f \in \hh} \rr_\ell(f) + \la \| f \|_\hh,
 \quad \la \ge 0 .
$$
If $\hh$ has kernel $ K : \xx \times \xx \to \R $ such that $ \sup_x K(x,x) \le \ka^2 $,
$\hh$ is continuously embedded in the space of bounded continuous functions on $\xx$,
with $ \| \cdot \|_\infty \le \ka \| \cdot \|_\hh $.
Hence, the uniform bounds \ref{it:fla-fl}, \ref{it:hfla-fla} may be derived from bounds in the RKHS norm.
The estimate $\widehat{f}_\la$ can be computed with a variety of methods,
such as empirical risk minimization (ERM)~\cite{scholkopf2002learning}, gradient descent (GD)~\cite{yao07} and stochastic gradient descent (SGD)~\cite{robbinsmonro}.

\begin{lem} \label{lem:marginfla}
Suppose \eqref{eq:hard-marginf} holds true with $ f = f_\ell $ and $ \de = \ga $.
Then, under the assumption \ref{it:fla-fl}, there is $\la_*$ such that
\eqref{eq:D-margin} holds true with $ f = f_\la $ and $ \de = \ga/2 $ for every $ \la \le \la_* $.
\end{lem}
\begin{proof}
 Let $ Df_\ell(x) = y_* = c_*(x) $ (recall that $\ell$ is Fisher consistent), and let
 $$
  a = \langle f_\la(x) , y_* \rangle - \max_{y \ne y_*} \langle f_\la(x) , y \rangle .
 $$
 Then
 $$
  a
  = \langle f_\ell(x) , y_* \rangle - \underbracket{\langle f_\ell(x) - f_\la(x) , y_* \rangle}_{b} - \underbracket{\max_{y \ne y_*} \langle f_\la(x) , y \rangle}_{c} ,
 $$
where
 $
  b \le \| f_\ell - f_\la \|_\infty ,
 $
 and
 \begin{align*}
  c & = \max_{y \ne y_*} ( \langle f_\ell(x) , y \rangle + \langle f_\la(x) - f_\ell(x) , y \rangle ) \\
  & \le \max_{y \ne y_*} \langle f_\ell(x) , y \rangle + \max_{y \ne y_*} \langle f_\la(x) - f_\ell(x) , y \rangle \\
  & \le \max_{y \ne y_*} \langle f_\ell(x) , y \rangle + \| f_\la - f_\ell \|_{\infty} .
 \end{align*}
 In view of \ref{it:fla-fl}, there is $\la_*$ such that $ \| f_\la - f_\ell \|_\infty \le \ga/4 $.
 Hence, for every $ \la \le \la_* $, using \eqref{eq:hard-marginf} we obtain
 \begin{align*}
  a
  & \ge \langle f_\ell(x) , y_* \rangle - \ga/4 - \max_{y \ne y_*} \langle f_\ell(x) , y \rangle - \ga/4 \\
  & = M(f_\ell(x)) - \ga/2 \ge \ga - \ga/2  = \ga/2 .
 \end{align*}
 This implies \eqref{eq:Df=c*}, and thus \eqref{eq:hard-marginf}, for $ f = f_\la $ and $ \de = \ga/2 $.
 The assertion now follows from \Cref{lem:D-margin}.
\end{proof}

\paragraph{Square loss.}

We can now state our first main result.
\begin{thm} \label{thm:main1}
 Suppose the hard-margin condition
 $$
  \min_{y \ne c_*(X)} \rho( c_*(X) \mid X )  - \rho( y \mid X ) \ge \de \quad \text{almost surely} .
 $$
Then, under the assumptions \ref{it:fla-fl}, \ref{it:hfla-fla}, there is $\la_*$ such that, for every $ \la \le \la_* $,
$$
 \E | \rr(D\widehat{f}) - \rr(Df') | \lesssim \exp(-n \de^2 \la / b^2) .
$$
\end{thm}
\begin{proof}
First, recall that, thanks to \Cref{lem:margin}, $ D\eta = c_* $
and $ M(\eta(X)) \ge \de $ almost surely.
Moreover, by \Cref{lem:marginfla} (and \Cref{lem:D-margin}), we have $ Df_\la = c_* $ and $ M(f_\la(X)) \ge \de/2 $ almost surely for $\la\le\la_*$.
Thus, \eqref{eq:misc} holds true,
and the claim follows from \Cref{prop:poly-exp}.
\end{proof}

\paragraph{Margin losses.}

We now consider surrogate losses of the form
$$
 \ell_\phi(w,y) = \phi( \langle w , y \rangle )
$$
for some scalar function $ \phi : \R \to [0,\infty) $.
We denote the minimizer of the corresponding risk $ \rr_\phi(f) = \E \phi( \langle f(X) , Y \rangle ) $ by
$$
 f_\phi = \argmin_{f \in L^0(\xx,\R^{T-1})} \rr_\phi (f) .
$$
Following and generalizing the analysis of \cite{Zhang2004,nitanda2019},
we want to extract an inner risk from $ \rr_\phi $.
The idea is to expand
$$
 \rr_\phi(f) = \E_X \sum_{y\in\yy} \phi( \langle f(X) , y 
 \rangle ) \rho(y \mid X)
$$
and isolate the argument of $ \E_X $ removing the dependence on $X$.
Recalling the definition of $ \Delta$ in \Cref{lem:margin}, we introduce the inner risk
$$
 \Phi(p,w) = \sum_{y\in\yy} \phi(\langle w , y \rangle)  p_y , \qquad p \in \Delta, w \in \R^{T-1} ,
$$
and the inner risk minimizer
$$
 h_\phi : \Delta \to \R^{T-1} ,
 \qquad
 h_\phi(p) = \argmin_{w\in\R^{T-1}} \Phi(p,w) .
$$
Note that, denoting $ p(x)_y = \rho(y \mid x) $, we have
\begin{equation} \label{eq:fx=hpx}
 f_\phi(x) = h_\phi(p(x)) .
\end{equation}
In the following, we will be assuming that
\begin{enumerate}[resume,label=\textnormal{(\roman*)}]
 \item $ \ell_\phi $ is Fisher consistent; \label{it:fisher}
 \item $ \langle h_\phi(p) , y \rangle $ is a non-decreasing function of $p_y$ . \label{it:hphipy}
\end{enumerate}
As previously mentioned, losses satisfying \ref{it:fisher} are indeed abundant.
For a general characterization of Fisher consistency in the framework of simplex encoded classification,
we refer to \cite{NIPS2012_1cecc7a7}.
The requirement \ref{it:hphipy} is easily met by many functions $\phi$,
as the next lemma shows.
Essentially, it is sufficient for the loss to be decreasing and convex.
Notable examples of $\phi$ satisfying both \ref{it:fisher} and \ref{it:hphipy} are the logistic loss $ \phi(t) = \ln(2)^{-1} \ln ( 1 + e^{-t} ) $,
and the exponential loss $ \phi (t) = e^{-t} $.
\begin{lem}
 Suppose $\phi$ is twice differentiable, non-increasing and convex.
 Then \ref{it:hphipy} holds true.
\end{lem}
\begin{proof}
 Let $ \Psi(p) = \nabla_w \Phi (p,h_\phi(p)) $.
 By definition of $ h_\phi $, we have $ \Psi(p) = 0 $,
 hence $ \J \Psi(p) = 0 $ as well.
Calculating the derivatives, we have
\begin{align*}
 0 = \J \Psi(p) & = \sum_{j\in\yy} \left( \phi''(\langle h_\phi(p) , j \rangle) \ j \otimes j \ \J h_\phi(p) \ p_j
 + \phi'(\langle h_\phi(p) , j \rangle) \ j \otimes e_j \right) ,
\end{align*}
 where $e_j$ denotes the vector of $\R^T$ with $ [e_j]_y = \delta_{j,y} $.
 Thus, for all $ y \in \yy $,
 \begin{align*}
  0 = \J \Psi(p) e_y = \sum_{j\in\yy} p_j \ \phi''(\langle h_\phi(p) , j \rangle) \ j \ \langle  \tfrac{\partial h_\phi(p)}{\partial p_y} , j \rangle
   + \phi'(\langle h_\phi(p) , y \rangle) \ y ,
 \end{align*}
 and therefore
 \begin{align*}
  0 = \langle \tfrac{\partial h_\phi(p)}{\partial p_y} , \J \Psi(p) e_y \rangle
  = \sum_{j\in\yy} p_j \ \phi''(\langle h_\phi(p) , j \rangle) \ \langle  \tfrac{\partial h_\phi(p)}{\partial p_y} , j \rangle^2
  + \phi'(\langle h_\phi(p) , y \rangle) \ \langle \tfrac{\partial h_\phi(p)}{\partial p_y} , y \rangle .
 \end{align*}
Since $ \phi'' \ge 0 $ and $ \phi' \le 0 $, we must have $ \langle \tfrac{\partial h_\phi(p)}{\partial p_y} , y \rangle \ge 0 $,
which proves the claim.
\end{proof}
In order to derive exponential rates for margin losses,
we need to transfer the hard-margin condition
from the conditional probabilities to the minimizer of the margin loss.
This is the content of the following lemma.
\begin{lem} \label{lem:m(de)}
 Suppose that \eqref{eq:hard-margin} holds true with $\de=\ga$.
 Then, under the assumption \ref{it:hphipy}, \eqref{eq:hard-marginf} holds true with $f=f_\phi$
 and $\de=m(\ga)$,
 where
 $$
  m(\ga) = \max_{y,j\in\yy} \min \{ M(h_\phi(p)) : p \in \Delta , p_y - p_j = 2\ga \} .
 $$
\end{lem}
\begin{proof}
 Let $ p(X)_y = \rho(y \mid X) $. By \eqref{eq:fx=hpx} we have
 $$
  M(f_\phi(X)) = M(h_\phi(p(X))) .
 $$
Let $ y, j \in \yy $ be such that
 $$
  M(h_\phi(p(X))) = \underbracket{\langle h_\phi(p(X)) , y \rangle}_{a} - \underbracket{\langle h_\phi(p(X)) , j \rangle}_{b} .
 $$
 In view of \eqref{eq:hard-margin} and \ref{it:hphipy}, there is $ p \in \Delta $ with $ p_y - p_j = 2\de $ such that $a$ decreases and $b$ increases, hence
 \begin{equation*}
  M(h_\phi(p(X))) \ge M(h_\phi(p)) .
 \end{equation*}
 Taking the minimum over such a $p$
 and the maximum over $y$ and $j$,
 we obtain the assertion.
\end{proof}
To visualize the lower bound $m(\ga)$, note that, for $T=2$,
it corresponds to $ \max \{ h_\phi(1/2+\ga) , -h_\phi(1/2-\ga) \} $
(cf.~with \cite{nitanda2019}).

We can finally prove our main result for margin losses.
\begin{thm} \label{thm:main2}
 Suppose the hard-margin condition
 $$
  \min_{y \ne c_*(X)} \rho( c_*(X) \mid X )  - \rho( y \mid X ) \ge \de \quad \text{almost surely} .
 $$
 Then, under the assumptions \ref{it:fla-fl}$\div$\ref{it:hphipy},
 there is $\la_*$ such that, for every $ \la \le \la_* $,
$$
 \E | \rr(D\widehat{f}) - \rr(Df') | \lesssim \exp(-n \ m(\de)^2 \la / b^2) ,
$$
where $m(\de)$ is defined in \Cref{lem:m(de)}.
\end{thm}
\begin{proof}
By assumption \ref{it:fisher}, we have $ Df_\phi = c_* $ almost surely.
Moreover, thanks to \Cref{lem:m(de)}, we have $ M(f_\phi(X)) \ge m(\de) $ almost surely.
Now, \Cref{lem:marginfla} (together with \Cref{lem:D-margin}) gives that
$ Df_\la = c_* $ and $ M(f_\la(X)) \ge m(\de)/2 $ almost surely for $ \la \le \la_* $.
Therefore, we have \eqref{eq:misc},
and \Cref{prop:poly-exp} yields the result.
\end{proof}

The critical value $\la_*$ in \Cref{thm:main1} and \Cref{thm:main2}
can be quantified in presence of additional assumptions on the distribution.
For example, consider the case of a kernel ridge regression estimator in a separable RKHS $\hh$.
Suppose that the kernel $K:\xx\times\xx\to\R$ is bounded by $\ka$,
and define the covariance operator as
$$
 T : \hh \to \hh , \qquad T = \int_\xx K_x \otimes K_x d\rho(x) ,
$$
where $ K_x = K(\cdot,x) $ and $\rho$ is the marginal distribution on $\xx$.
Further, suppose there exist $ g \in \hh $ and $ s \in (0,1/2] $ such that
$$
f_\ell = T^s g .
$$
This is known as the \emph{source} condition,
and it corresponds to assuming Sobolev smoothness of the regression function.
Then, it can be proved that \cite{caponnetto2007}
$$
 \| f_\la - f_\ell \|_\hh \le \la^s \| g \|_\hh .
$$
As a consequence, $ \la_* $ in \Cref{lem:marginfla}, and therefore in \Cref{thm:main1}, may be picked as $ \la_* = (\de/4\ka\|g\|_\hh)^{1/s} $.

We finally remark that analogous results to \Cref{thm:main1} and \Cref{thm:main2}
may be proved under the soft-margin condition \eqref{eq:margin},
using the polynomial bound of \Cref{prop:poly-exp}.

\section{Experiments}
\label{sec:experiments}

This section is concerned with empirically verifying the theoretical analysis presented in Section~\ref{sec:analysis}. We will first consider a classification problem where the true function satisfies the hard-margin condition (defined in Eq.~\eqref{eq:hard-marginf}), and show how -- under optimization of a surrogate loss by gradient descent -- the misclassification loss decreases more quickly than the surrogate loss.
Then we will take into account a different synthetic dataset, where the weaker soft-margin or low noise condition (see Eq.~\eqref{eq:marginf}) is satisfied. We will verify how the rate of change of the misclassification error with the number of points in the dataset adheres to the theoretical rates.

\begin{figure}[h]
	\centering
	\includegraphics[width=0.5\linewidth]{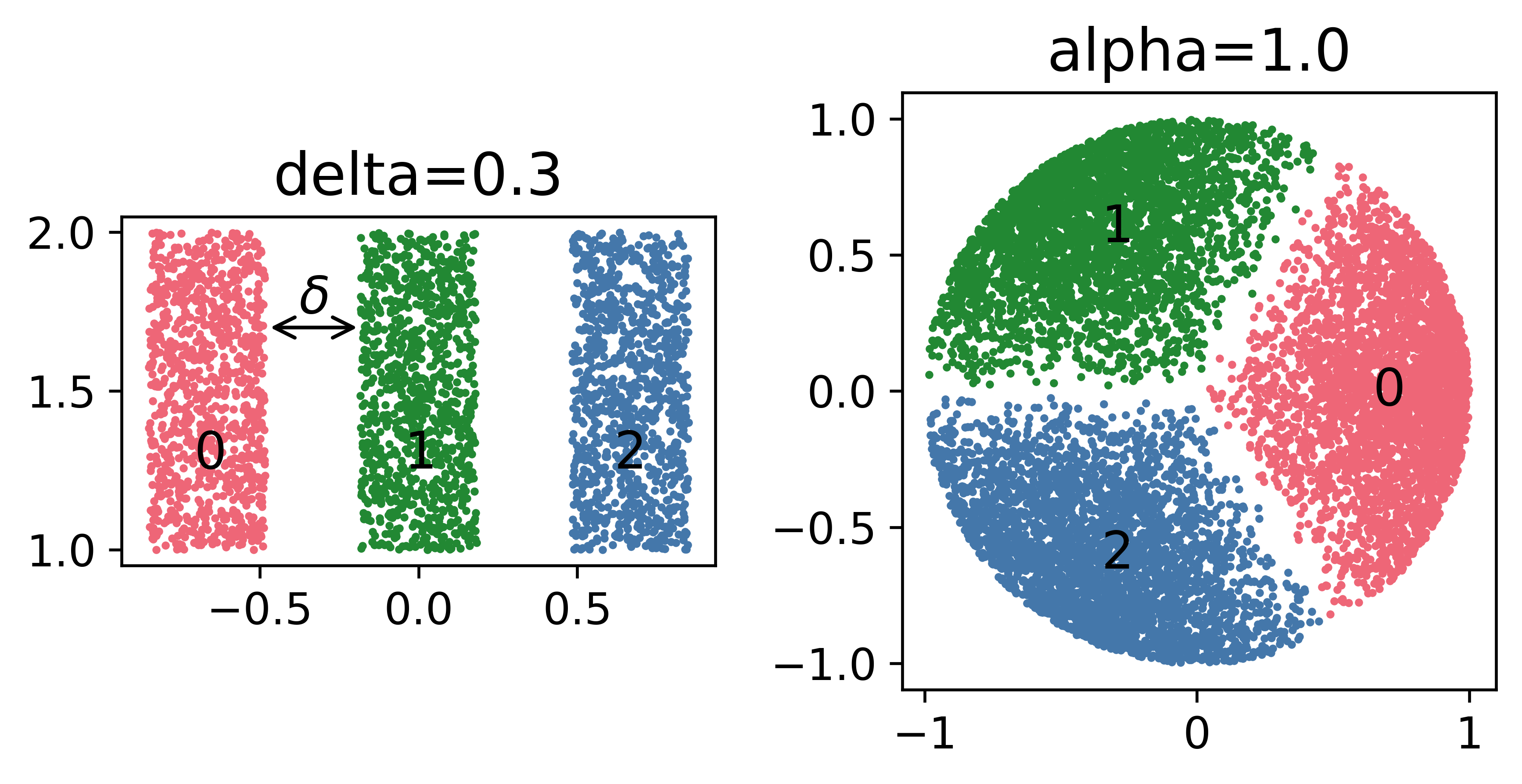}
	\caption{Sample datasets. In the left panel, the three classes are separated by a hard margin of length $\delta$. In the right panel there is no hard margin, but the probability of a point falling close to the boundary is decreasing (soft-margin).}
	\label{fig:data}
\end{figure}

Initially we compare three different surrogate loss functions: the logistic, the exponential and the square loss. We generated data in two dimensions such that the hard-margin condition holds with margin $\delta$, see Figure~\ref{fig:data}, left panel for a sample dataset. A random Fourier features (RFF) model~\cite{rahimi08} approximates potentially infinite dimensional feature maps in a reproducing kernel Hilbert space (RKHS) using finite dimensional randomized maps: given a kernel function $k(x, x') = \langle\psi(x), \psi(x')\rangle_{\mathcal{H}}$ the feature map $\psi\in\mathcal{H}$ can be approximated with function $z: \mathbb{R}^D \rightarrow \mathbb{R}^R$ such that $\langle \psi(x), \psi(x')\rangle_{\mathcal{H}} \approx \langle z(x), z(x') \rangle_{\mathbb{R}}$. Finally $z(x)$ can be used instead of the sample itself in a linear model with parameters $w\in\mathbb{R}^R$: $f(x) = w^\top z(x)$.
In order to learn the parameters $w$ we minimize the regularized surrogate loss with gradient descent.
In Figure~\ref{fig:losses} we plot the 0-1 error, as well as the surrogate losses on unseen data as a function of the optimization epoch. A separate model was trained for each of the three surrogates 20 times with a new synthetic dataset. The intuition behind exponential rates in hard-margin classification can be verified by noting how the 0-1 loss converges at a much faster pace than the surrogate: from another perspective, when the 0-1 loss is zero the surrogate loss can still decrease for many epochs. We can further notice how not all surrogates are equal: for both the small ($\delta = 0.1$) and the larger margin ($\delta = 0.2$), the square loss leads to faster convergence of the 0-1 error than both exponential and logistic losses.

\begin{figure}[h]
	\centering
	\includegraphics[width=0.7\linewidth]{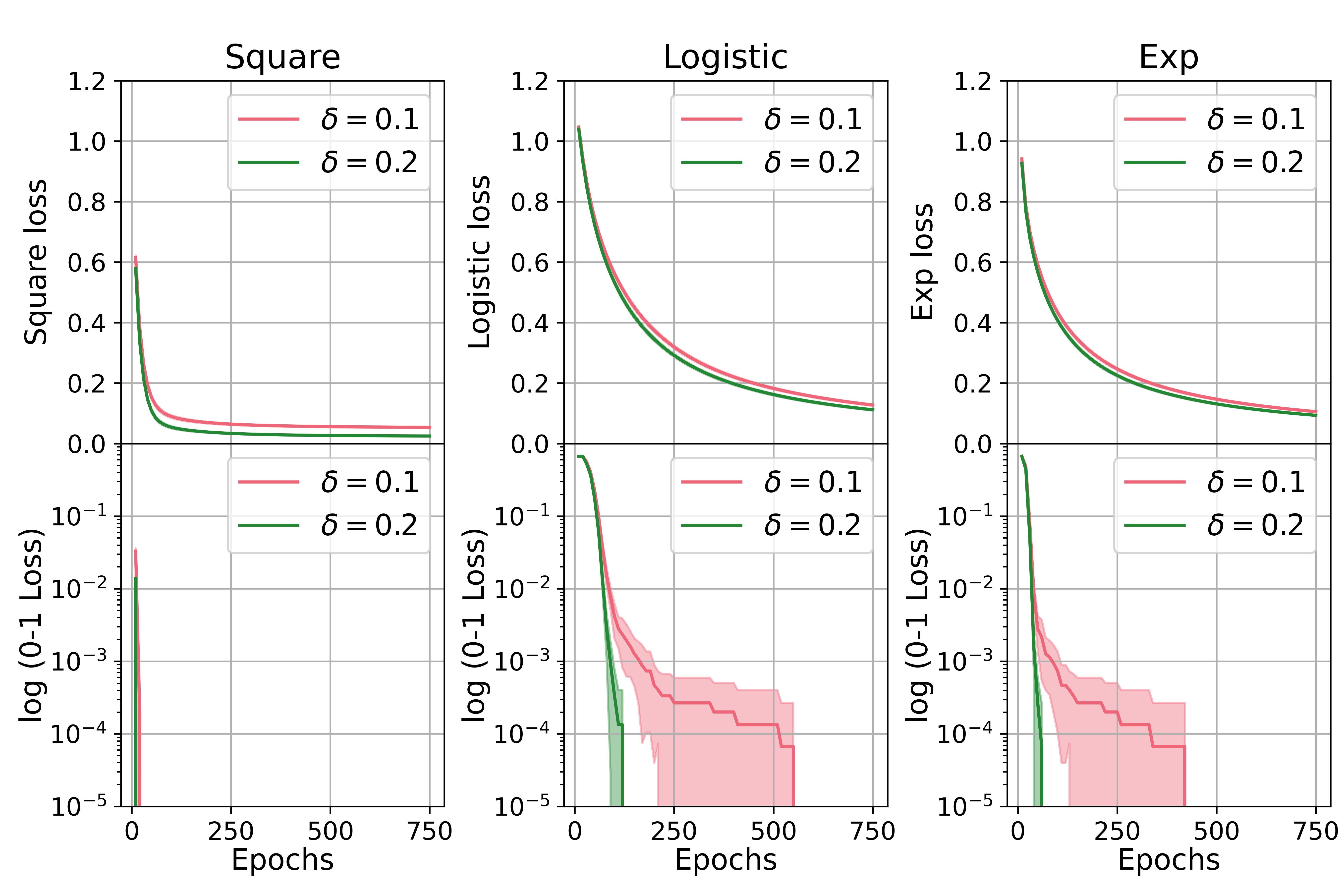}
	\caption{Optimization curves on hard-margin classification with different surrogate losses. Each panel contains two curves calculated on datasets with different margins $\delta$. The top row shows the surrogate loss, the bottom row shows the 0-1 loss.}
	\label{fig:losses}
\end{figure}

%

For the second experiment, we generated a synthetic dataset in two dimensions and with three classes such that the probability of a point falling close to the decision boundary decreases with the distance to the boundary itself as $M^\alpha$ for margins $M$ smaller than 1 (see \eqref{eq:marginf} and \Cref{fig:data}, right panel). We then used a linear model, trained by minimizing the regularized logistic loss with gradient descent until convergence. 
We repeated the experiment 100 times for datasets generated with five different values of $\alpha$ (a higher $\alpha$ results in an easier problem), and an increasing number of points, and recorded the average 0-1 loss over unseen data. 
We then plot the 0-1 loss against the number of points for each value of $\alpha$, and observe that the trends are approximately linear on a log-log plot (see Figure~\ref{fig:soft-rates}). We fit a straight line for each $\alpha$, and look at how the slope of this line changes with $\alpha$. From Proposition~\ref{prop:poly-exp} we expect the error to drop more rapidly with higher $\alpha$; in particular the rate of decrease is predicted to be $n^{-\alpha/2}$ ignoring constant and logarithmic factors. By plotting the slopes of the error rates we obtain a straight line with slope $-0.35$, which is close to the prediction of $-0.5$ (see the inset on Figure~\ref{fig:soft-rates}).

\begin{figure}[h]
	\centering
	\includegraphics[width=0.5\linewidth]{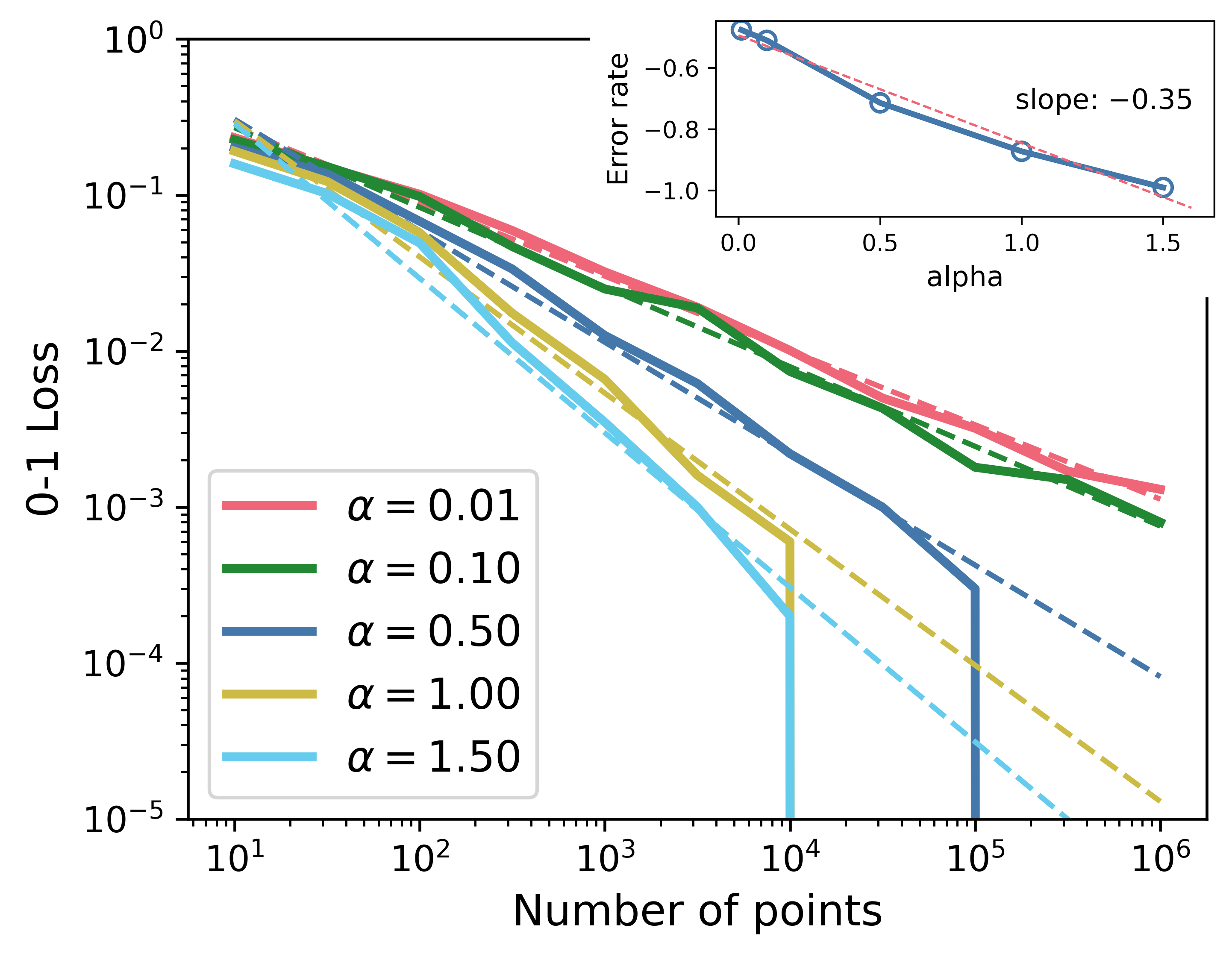}
	\caption{\emph{Main figure}: error rates for multiclass classification with polynomial soft-margin with increasing dataset size. \emph{Inset}: Linear rate of convergence of the error with $\alpha$.}
	\label{fig:soft-rates}
\end{figure}

\section{Conclusions}
\label{sec:conclusions}

In this paper we have shown how, under the hard-margin condition and
for a very general framework which encompasses many different models and surrogate losses, 
the multiclass classification error exhibits exponentially fast convergence.
Along the way we have provided an error decomposition where the bias term disappears.
This kind of result fits with the recent empirical observations of how even highly 
overparametrized models do not overfit the training data.
Our analysis can be experimentally verified for several losses, 
and different margin conditions.

Several possible extensions of this work have been left for future work.
Beyond the hard-margin and low-noise conditions,
robustness with respect to different kinds of noise may be studied.
The explicit application of our bounds to specific models --
which was sketched in this paper for kernel ridge regression --
could be especially interesting for (deep) neural networks, 
for which fast convergence on classification problems has been ascertained.
Indeed, for the latter models, the interplay of exponential convergence
and overparameterization is a further topic of great interest.



\printbibliography

\end{document}